\documentclass{tlp}

\usepackage{graphicx}
\usepackage{latexsym}
\usepackage{amssymb}
\usepackage{amsmath}
\usepackage{stmaryrd}
\usepackage{color}
\usepackage{hyperref}

\def\PP{{\cal P}}

\def\MM{{\cal MM}}
\def\SM{{\cal SM}}

\def\.{.}
\def\dots{\.\.\.}
\def\wrt{w.r.t.\ }
\def\<{\langle}
\def\>{\rangle}
\def\product{\cdot}

\newtheorem{definition}{Definition}
\newtheorem{example}{Example}
\newtheorem{lemma}{Lemma}
\newtheorem{theorem}{Theorem}

\def\natural{integer}
\def\nvar{\natural\ variable}
\def\nvars{\natural\ variables}
\def\lvar{logical variable}
\def\lvars{logical variables}

\newcommand{\ssize}[1]{||#1||}
\newcommand{\lfp}[1]{\mathit{T}_{#1}}
\newcommand{\ex}[2]{#1^{#2}}

\def\true{\mathit{true}}
\def\false{\mathit{false}}

\def\calc{{\cal C}}
\def\locally{rule-}
\def\Lb{Rule-bounded}
\def\LB{Rule-bounded}
\def\lb{rule-bounded}
\def\Gb{Cycle-bounded}
\def\gb{cycle-bounded}
\def\GB{Cycle-bounded}

\def\sb{size-bounded}
\def\Sb{Size-bounded}

\def\csb{cycle-size-bounded}

\def\PTIME{\mathit{PTIME}}
\def\NP{\mathit{NP}}

\def\co{\mathit{co}}

\def\tsize{total size}
\def\ts{tsize}

\def\ug{firing graph}
\def\Ug{Firing graph}
\newcommand{\unig}[1]{\Omega(#1)} 

\def\rbody{rbody}

\def\sbody{sbody}
\def\srbody{srbody}
\def\linear{\ell}

\def\relevant{relevant}

\def\ar{\mathit{arity}}

\newcommand{\nop}[1]{{}}

\submitted{30 December 2014}
\revised{8 July 8 2015}
\accepted{14 December 2015}

\begin{document}
\bibliographystyle{acmtrans}

\title{Using Linear Constraints for Logic Program Termination Analysis}
%
%

\author[M. Calautti, S. Greco, C. Molinaro and I. Trubitsyna]{
MARCO CALAUTTI, SERGIO GRECO,  \\
{\normalsize \em CRISTIAN MOLINARO, IRINA TRUBITSYNA}\\
DIMES, Universit\`{a} della Calabria\\
87036 Rende(CS), Italy\\
E-mail: \{calautti,greco,cmolinaro,trubitsyna\}@dimes.unical.it
}



\maketitle

\label{firstpage}

\begin{abstract}
It is widely acknowledged that function symbols are an important feature in answer set programming, as they make modeling easier, increase the expressive power, and allow us to deal with infinite domains.
The main issue with their introduction is that the evaluation of a program might not terminate and checking whether it terminates or not is undecidable.
To cope with this problem, several classes of logic programs have been proposed where the use of function symbols is restricted but the program evaluation termination is guaranteed. 
Despite the significant body of work in this area, current approaches do not include many simple practical programs whose evaluation terminates.
In this paper, we present the novel classes of \emph{\lb} and \emph{\gb\ programs}, which overcome different limitations of current approaches by performing a more global analysis of how terms are propagated from the body to the head of rules. 
Results on the correctness, the complexity, and the expressivity of the proposed approach are provided. 

\noindent
Under consideration in Theory and Practice of Logic Programming (TPLP).
\end{abstract}

\begin{keywords}
Answer set programming, function symbols, bottom-up evaluation, program evaluation termination, stable models
\end{keywords}

\section{Introduction}

Enriching answer set programming with function symbols has recently seen a surge in interest.
Function symbols make modeling easier, increase the expressive power, and allow us to deal with infinite domains.
At the same time, this comes at a cost: common inference tasks (e.g., cautious and brave reasoning) become undecidable.

Recent research has focused on identifying classes of logic programs that impose some limitations on the use of function symbols but guarantee decidability of common inference tasks.
Efforts in this direction are the class of \emph{finitely-ground} programs~\cite{CalimeriCIL08} 
and the more general class of \emph{bounded term-size} programs~\cite{RiguzziS13}.
Finitely-ground programs have a finite number of stable models, each of finite size,
whereas bounded term-size (normal) programs have a finite well-founded model. 
Unfortunately, checking if a logic program is bounded term-size or even finitely-ground is semi-decidable.

Considering the stable model semantics, decidable subclasses of finitely-ground programs have been proposed.
These include the classes of \emph{$\omega$-restricted programs}~\cite{Syrjanen01},
\emph{$\lambda$-restricted programs}~\cite{GebserST07},
\emph{finite domain programs}~\cite{CalimeriCIL08},
\emph{argument-restricted programs}~\cite{LierlerL09},
\emph{safe} and \emph{$\Gamma$-acyclic programs}~\cite{GrecoST12iclp,TPLP14},
\emph{mapping-restricted programs}~\cite{CalauttiPPDP13}, and 
\emph{bounded programs}~\cite{GrecoMT13ijcai}.
The above techniques, that we call \emph{termination criteria}, provide (decidable) sufficient conditions for a program to be finitely-ground.

Despite the significant body of work in this area, there are still many simple practical programs whose evaluation terminates but this is not detected by any of the current termination criteria.
Below is an example.

\begin{example}\label{ex:bubble}
Consider the following program $\PP_{\ref{ex:bubble}}$ implementing the bubble sort algorithm:
\[
  \begin{array}{l}
   r_0:\ \tt bub(L,[\,],[\,]) \leftarrow \tt input(L).\\
   r_1:\ \tt bub([Y|T],[X|Cur],Sol)  \leftarrow \tt bub([X|[Y|T]],Cur,Sol),X \leq Y.\\
   r_2:\ \tt bub([X|T],[Y|Cur],Sol)  \leftarrow  \tt bub([X|[Y|T]],Cur,Sol),Y<X.\\
   r_3:\ \tt bub(Cur,[\,],[X|Sol])  \leftarrow  \tt bub([X|[\,]],Cur,Sol).
  \end{array}
\]
\noindent
The list to be sorted is given by means of a fact of the form $\tt input([a_1,\dots,a_n])$.
The bottom-up evaluation of this program always terminates for any input list. 
The ordered list $\tt Sol$ can be obtained from the atom $\tt bub([\,],[\,],Sol)$ in the program's minimal model.~\hfill$\Box$
\end{example}

Although the bottom-up evaluation of $\PP_{\ref{ex:bubble}}$ always terminates for any input list, none of the termination criteria in the literature is able to realize it. 
One problem with them is that when they analyze how terms are propagated from the body to the head of rules, they look at arguments \emph{individually}.
For instance, in rule $r_1$ above, the simple fact that the second argument of $\tt bub$ has a size in the head greater than the one in the body prevents several techniques from realizing termination of the bottom-up evaluation of $\PP_{\ref{ex:bubble}}$.
More general classes such as mapping-restricted and bounded programs are able to do a more complex (yet limited) analysis of how some groups of arguments affect each other.
Still, all current termination criteria are not able to realize that in every rule of $\PP_{\ref{ex:bubble}}$ the \emph{overall} size of the terms in the head does not increase w.r.t. the \emph{overall} size of the terms in the body.
One of the novelties of the technique proposed in this paper is the capability of doing this kind of analysis, thereby identifying programs (whose evaluation terminates) that none of the current techniques include.

The technique proposed in this paper easily realizes that the bottom-up evaluation of $\PP_{\ref{ex:bubble}}$ always terminates for any input list.
In particular, this is done using linear constraints which measure the size of terms and atoms in order to check if the rules' head sizes are bounded by the size of some body atom when propagation occurs.
Thus, our technique can understand that, in every rule, the overall size of the terms in the body does not increase during their propagation to the head, as there is only a simple redistribution of terms.
Many practical programs dealing with lists and tree-like structures satisfy this property---below are two examples. 
However, our technique is not limited only to this kind of programs.

\begin{example}\label{ex:visit}
Consider the program $\PP_{\ref{ex:visit}}$ below, performing a depth-first traversal of an input tree:
 \[
  \begin{array}{l}
   r_0:\ \tt visit(Tree,[\,],[\,]) \leftarrow input(Tree).\\
   r_1:\ \tt visit(Left,[Root |Visited],[Right|ToVisit]) \leftarrow \\
   \hspace*{42mm} \tt visit(tree(Root,Left,Right), Visited,ToVisit).\\
   r_2:\ \tt visit(Next, Visited,ToVisit) \leftarrow  visit(null, Visited,[Next|ToVisit]).
  \end{array}
 \]
The input tree is given by means of a fact of the form  $\tt input(tree(value,left,right))$  
where $\tt tree$ is a ternary function symbol used to represent tree structures. The program visits the nodes of the tree and puts them in a list following 
a depth-first search. 
The list $\tt L$ of visited elements can be obtained from the atom $\tt visit(null,L,[\ ])$ in the program's minimal model.
For instance, if the input tree is
$$
\tt input(tree(a,tree(c,null,tree(d,null,null)), tree(b,null,null))).
$$
\noindent
the program produces the list $\tt [b,d,c,a]$ containing the nodes of the tree in opposite order w.r.t. the traversal.~\hfill$\Box$
\end{example}

Also in the case above, even if the program evaluation terminates for every input tree, none of the currently known techniques is able to 
detect it, while the technique proposed in this paper does.

\begin{example}\label{ex:append}
Consider the following program $\PP_{\ref{ex:append}}$ computing the concatenation of two lists: 
 \[
  \begin{array}{lll}
   r_0:\ \tt reverse(L_1,[\,]) & \leftarrow & \tt input1(L_1).\\
   r_1:\ \tt reverse(L_1,[X|L_2]) & \leftarrow & \tt reverse([X|L_1],L_2).\\
   r_2:\ \tt append(L_1,L_2) & \leftarrow & \tt reverse([\,],L_1),\ input2(L_2). \\
   r_3:\ \tt append(L_1,[X|L_2]) & \leftarrow & \tt append([X|L_1],L_2).
  \end{array}
 \]
\noindent
Here $\tt input1$ and $\tt input2$ are used to store the lists $\tt L_1$ and $\tt L_2$ to be concatenated.
The result list $\tt L$ can be retrieved from the atom $\tt append([\,],L)$ in the minimal model of $\PP_{\ref{ex:append}}$.
Clearly, the bottom-up evaluation of the program always terminates.~\hfill$\Box$
\end{example}

We point out that the problem of detecting decidable classes of programs is relevant not only from a theoretical point of view, as real applications make use of structured data and functions symbols (e.g., lists, sets, bags, arithmetic). 
Classical applications need the use of  structured data such as bill of materials consisting in the description of all items that compose a product, down to the lowest level of detail~\cite{CeriGT90}, management of strings in bioinformatics applications, managing and querying ontological data using logic languages~\cite{CaliGLMP10,ChaudhriHTW13}, as well as applications based on greedy and dynamic programming algorithms~\cite{GrecoZG92,Greco99}.

\vspace*{1mm}
\noindent {\bf Contribution.}
We propose novel techniques for checking if the evaluation of a logic program  terminates (clearly, we define sufficient conditions). 
Our techniques overcome several limitations of current approaches being able to perform a more global analysis of how terms are propagated from the body to the head of rules. 
To this end, we use linear constraints to measure and relate the size of head and body atoms.
We first introduce the class of \emph{\lb} programs, which looks at individual rules, and then propose the class of \emph{\gb} programs, which relies on the  analysis of groups of rules.
We show the correctness of the proposed techniques and provide upper bounds on their complexity.
We also study the relationship between the proposed classes and current termination criteria. 

\vspace*{1mm}
\noindent {\bf Organization.}
Section~\ref{sec:preliminaries} reports preliminaries on logic programs with function symbols. 
Sections~\ref{sec:locally_bounded} introduces the class of \lb\ programs, whereas Section~\ref{sec:correct-express-lb}
presents several theoretical results on its correctness and expressivity.
Section~\ref{sec:mutual_recursion} introduces the class of \gb\ programs along with results on its correctness and expressivity.
The complexity analysis is addressed in Section~\ref{sec:complexity}.
Related work and conclusions are reported in Sections~\ref{sec:related-work}~and~\ref{sec:conclusions}, respectively.

\section{Preliminaries}\label{sec:preliminaries}
This section recalls 
syntax and the stable model semantics of logic programs with function symbols~\cite{GelLif88,2012Gebser}.

\vspace*{1mm}
\noindent {\bf Syntax.}
We assume to have (pairwise disjoint) infinite sets of \emph{\lvars}, \emph{predicate symbols}, and \emph{function symbols}.
Each predicate and function symbol $g$ is associated with an \emph{arity}, denoted $\ar(g)$, which is a non-negative integer. Function symbols of arity 0
are called \emph{constants}.
Variables appearing in logic programs are called ``\lvars'' and will be denoted by upper-case letters in order to distinguish them from variables appearing in linear constraints, which are called ``\nvars'' and will be denoted by lower-case letters.
A \emph{term} is either a \lvar, or an expression of the form $f(t_1,\dots,t_m)$, where $f$ is a function symbol of arity $m \ge 0$ and $t_1,\dots,t_m$ are terms.

An \emph{atom} is of the form $p(t_1,\dots,t_n)$, where $p$ is a predicate symbol of arity $n \ge 0$ and $t_1,\dots,t_n$ are terms. 
A \emph{literal} is an atom $A$ (\emph{positive} literal) or its negation $\neg A$ (\emph{negative} literal).

A \emph{rule} $r$ is of the form
$
A_1 \vee \dots \vee A_m \leftarrow B_1,\dots, B_k, \neg C_1,\dots, \neg C_n
$,
where $m >0$,  $k\geq 0$, $n \geq 0$, and $A_1,\dots ,A_m,B_1,\dots,B_k,$ $C_1,$ $\dots,C_n$ are atoms. 
The disjunction $A_1 \vee \dots \vee A_m$ is called the \emph{head} of $r$ and is denoted by $head(r)$. 
The conjunction $B_1,\dots, B_k, \neg C_1,\dots, \neg C_n$ is called the \emph{body} of $r$ and is denoted by $body(r)$.
With a slight abuse of notation, we sometimes use $body(r)$ (resp. $head(r)$) to also denote the \emph{set} of literals appearing in the body (resp. head) of $r$. 
If $m=1$, then $r$ is {\em normal}; in this case, $head(r)$ denotes the head atom. 
If $n=0$, then $r$ is {\em positive}. 

A \emph{program} is a finite set of rules.
A program is \emph{normal} (resp. \emph{positive}) if every rule in it is normal (resp. positive).
We assume that programs are \emph{range restricted}, i.e., for every rule, every \lvar\ appears in some positive body literal.
W.l.o.g., we also assume that different rules do not share \lvars.

A term (resp. atom, literal, rule, program) is {\em ground} if no \lvars\ occur in it.
A ground normal rule with an empty body is also called a \emph{fact}. 
A predicate symbol $p$ is \emph{defined by} a rule $r$ if $p$ appears in the head of $r$.

A \emph{substitution} $\theta$ is of the form $\{X_1/t_1,\dots,X_n/t_n\}$, where  $X_1,\dots,X_n$ are distinct \lvars\ and $t_1,\dots,t_n$ are terms. 
The result of applying $\theta$ to an atom (or term) $A$, denoted $A \theta$, is the atom (or term) obtained from $A$ by simultaneously replacing each occurrence of a \lvar\ $X_i$ in $A$ with $t_i$ if $X_i/t_i$ belongs to $\theta$.
Two atoms $A_1$ and $A_2$ \emph{unify} if there exists a substitution $\theta$, called a \emph{unifier} of $A_1$ and $A_2$, such that $A_1\theta =A_2\theta$.
The \emph{composition} of two substitutions $\theta=\{X_1/t_1,\dots,X_n/t_n\}$ and $\vartheta=\{Y_1/u_1,\dots,Y_m/u_m\}$, denoted $\theta\circ\vartheta$, is the substitution obtained from the set $\{X_1/t_1\vartheta,\dots,X_n/t_n\vartheta,$ $Y_1/u_1,\dots,Y_m/u_m\}$ by removing every $X_i/t_i\vartheta$ such that $X_i=t_i\vartheta$ and every $Y_j/u_j$ such that $Y_j \in \{X_1,\dots,X_n\}$.
A substitution $\theta$ is \emph{more general} than a substitution $\vartheta$ if there exists a substitution $\eta$ such that $\vartheta=\theta\circ\eta$.
A unifier $\theta$ of $A_1$ and $A_2$ is called a \emph{most general unifier} (mgu) of $A_1$ and $A_2$ if it is more general than any other unifier of $A_1$ and $A_2$
(indeed, the mgu is unique modulo renaming of \lvars).

\vspace*{1mm}
\noindent {\bf Semantics.}
Consider a program $\PP$.
The \emph{Herbrand universe} $H_{\PP}$ of $\PP$ is the
possibly infinite set of ground terms  
constructible using
function symbols (and thus, also constants) appearing in $\PP$. The \emph{Herbrand base}
$B_{\PP}$ of $\PP$ is the set of ground atoms 
constructible using predicate symbols appearing in $\PP$ and ground terms
of $H_{\cal P}$.

A rule (resp. atom) $r'$ is a {\em ground instance} of a rule (resp. atom) $r$ in $\PP$ if $r'$ can be obtained from $r$ by substituting every \lvar\ in $r$ with some ground term in $H_{\mathcal P}$.
We use $ground(r)$ to denote the set of all ground instances of $r$ and define $ground(\PP)$ to denote the set of all ground instances of the rules in $\PP$, i.e., $ground(\PP)=\cup_{r\in\PP} ground(r)$.

An \emph{interpretation} of $\PP$ is any subset $I$ of $B_{\cal P}$.
The truth value of a ground atom $A$ \wrt $I$, denoted $value_I(A)$, is $\true$ if $A \in I$, $\false$ otherwise.
The truth value of $\neg A$ \wrt $I$, denoted $value_I(\neg A)$, is $\true$ if $A \not\in I$, $\false$ otherwise.
A ground rule $r$ is {\em satisfied} by $I$, denoted $I\models r$, if there is a ground literal $L$ in $body(r)$ s.t. $value_I(L)=\false$ or there is a ground atom $A$ in $head(r)$ s.t. $value_I(A)=\true$. 
Thus, if the body of $r$ is empty, $r$ is satisfied by $I$ if there is an atom $A$ in $head(r)$ s.t. $value_I(A)\!=\!\true$.
An interpretation of $\PP$ is a \emph{model} of $\PP$ if it satisfies every ground rule in $ground(\PP)$.
A model $M$ of $\PP$ is minimal if no proper subset of $M$ is
a model of $\PP$.
The set of minimal models of $\PP$ is denoted by $\MM(\PP)$.

Given an interpretation $I$ of $\PP$, let $\PP^I$ denote the ground positive program derived
from $ground(\PP)$ by \emph{(i)} removing every rule containing a
negative literal $\neg A$ in the body with $A \in I$, and \emph{(ii)}
removing all negative literals from the remaining rules.
An interpretation $I$ is a \emph{stable model} of $\PP$ if $I \in \MM(\PP^I)$.
The set of stable models of $\PP$ is denoted by $\SM(\PP)$.
It is well known that stable models are minimal models (i.e., $\SM(\PP) \subseteq \MM(\PP)$), and $\SM(\PP) = \MM(\PP)$ for positive programs.

A positive normal program $\PP$ has a unique minimal model, which, with a slight abuse of notation, we denote as $\MM(\PP)$.
The \emph{immediate consequence operator} of $\PP$ is a function $\lfp{\PP}:2^{B_{\PP}}\rightarrow 2^{B_{\PP}}$ defined as follows: for every interpretation $I$, $\lfp{\PP}(I) = \{A\mid A \leftarrow B_1,\dots,B_n \in ground(\PP) \mbox{ and } \{B_1,\dots,B_n\} \subseteq I\}$.
The $i$-th iteration of $\lfp{\PP}$ ($i \geq 1$) w.r.t. an interpretation $I$ is defined as follows: $\lfp{\PP}^1(I)=\lfp{\PP}(I)$ and $\lfp{\PP}^i(I)=\lfp{\PP}(\lfp{\PP}^{i-1}(I))$ for $i > 1$.
The minimal model of $\PP$ coincides with $\lfp{\PP}^\infty(\emptyset)$.

\vspace*{1mm}
\noindent {\bf Finite programs.}
A program $\PP$ is said to be \emph{finite} under stable model semantics 
if, for every finite set of facts $D$, the program $\PP \cup D$ admits a finite number of stable models and each is of finite size, that is, $|\SM(\PP \cup D)|$ is finite and every stable model $M \in \SM(\PP \cup D)$ is finite.

Equivalently, a positive normal program $\PP$ is \emph{finite} if for every finite set of facts $D$,
there is a finite natural number $n$ such that $\lfp{\PP \cup D}^n(\emptyset)=\lfp{\PP \cup D}^\infty(\emptyset)$.
We call such programs \emph{terminating}.
In this paper we study new conditions under which a positive normal program $\PP$ is terminating. It is worth mentioning that
such conditions can be easily extended to general programs. This will be shown in the next section. 

\section{\LB\ Programs}\label{sec:locally_bounded}

In this section, we present \emph{\lb\ programs}, a class of programs whose evaluation always terminates and for which checking membership in the class is decidable.
Their definition relies on a novel technique which uses linear inequalities to measure terms and atoms' sizes and checks if the size of the head of a rule is always bounded by the size of a mutually recursive body atom (we will formally define what ``mutually recursive'' means in Definition~\ref{def:mutually-recursive-atom} below).

For ease of presentation, we restrict our attention to positive normal programs.
However, our technique can be applied to an arbitrary program $\PP$ with disjunction in the head and negation in the body by considering a positive normal program $st(\PP)$ derived from $\PP$ as follows.
Every rule $A_1\vee \dots \vee A_m\leftarrow body$ in $\PP$ is replaced with $m$ positive normal rules of the form $A_i \leftarrow body^{+}$ ($1 \leq i \leq m$) where $body^{+}$ is obtained from $body$ by deleting all negative literals.
In fact, the minimal model of $st(\PP)$ contains every stable model of $\PP$~\cite{GrecoST12iclp}---whence, the termination
of $st(\PP)$, which implies finiteness and computability of the minimal model will also imply that $\PP$ has a finite number of stable models, each of finite size, which can be computed.
In the rest of the paper, a program is understood to be positive and normal. We start by introducing some preliminary notions. 

\begin{definition}[\Ug]\label{def:unification-graph}
The \emph{\ug} of a program $\PP$, denoted $\unig{\PP}$, is a directed graph whose nodes are the rules in $\PP$ and such that there is an edge $\<r,r'\>$ if there exist two (not necessarily distinct) rules $r,r'\in\PP$ s.t. $head(r)$ and an atom in $body(r')$ unify.~\hfill$\Box$
\end{definition}

Intuitively, an edge $\<r,r'\>$ of $\unig{\PP}$ means that rule $r$ may cause rule $r'$ to ``fire''. 
The \ug\ of program $\PP_{\ref{ex:bubble}}$ of Example~\ref{ex:bubble} is depicted in Figure~\ref{fig:activation}.
In the definition above, when $r=r'$ we assume that $r$ and $r'$ are two ``copies'' that do not share any \lvar.

\begin{figure}[t!]
  \centering
 \includegraphics[height=2.4cm]{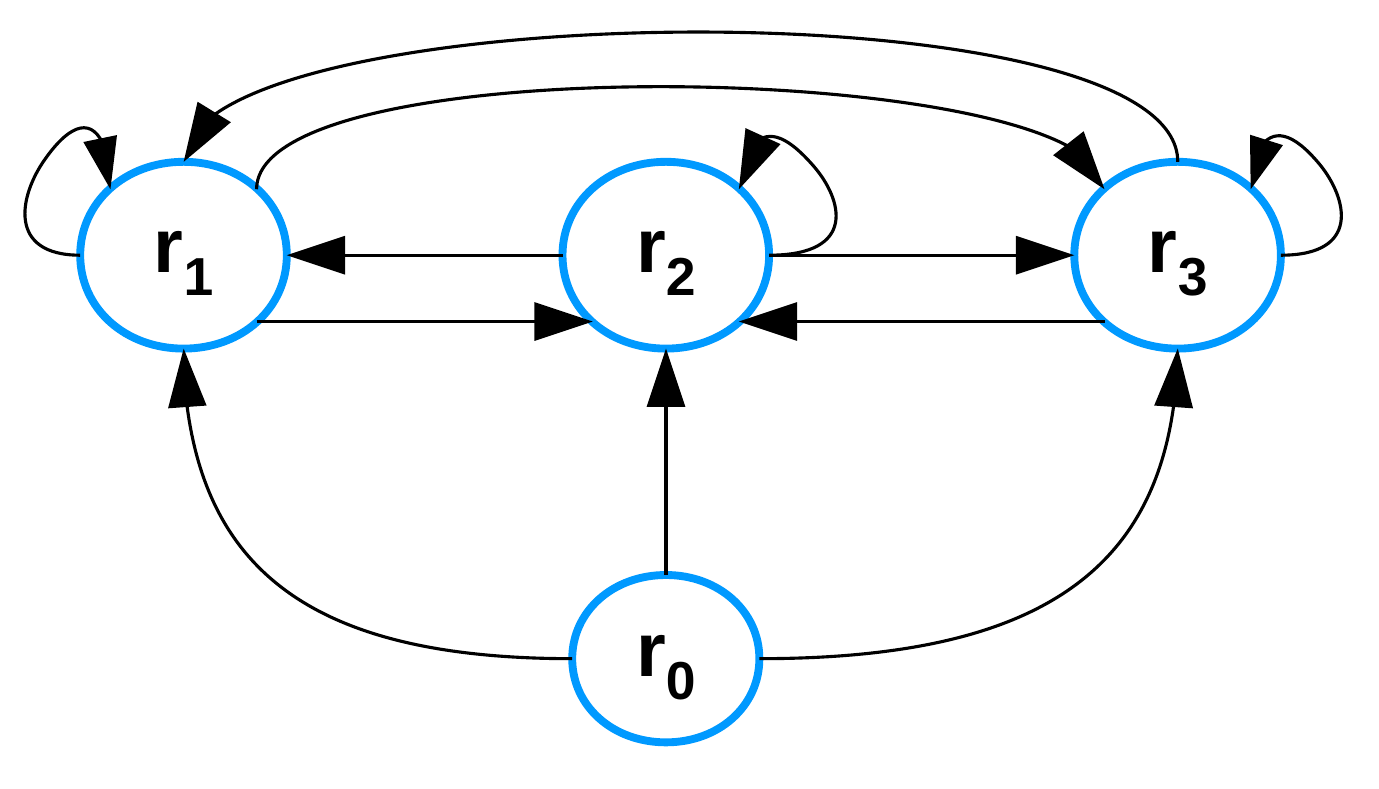}
   \caption{\Ug\ of~$\PP_{\ref{ex:bubble}}$.}
   \label{fig:activation}
\end{figure}

We say that a rule $r$ \emph{depends on} a rule $r'$ if $r$ can be reached from $r'$ through the edges of $\unig{\PP}$.
A \emph{strongly connected component} (SCC) of a directed graph $G$ is a maximal set $\calc$ of nodes of $G$ s.t. every node of $\calc$ can be reached from every node of $\calc$ (through the edges in $G$).
We say that an SCC $\calc$ is \emph{non-trivial} if there exists at least one edge in $G$ between two not necessarily distinct nodes of $\calc$.
For instance, the \ug\ in Figure~\ref{fig:activation} has two SCCs, $\calc_1=\{r_0\}$ and $\calc_2=\{r_1,r_2,r_3\}$, but only $\calc_2$ is non-trivial.
Given a program $\PP$ and an SCC $\calc$ of $\unig{\PP}$, $pred(\calc)$ 
denotes the set of predicate symbols defined by the rules in $\calc$.
We now define when the head atom and a body atom of a rule are mutually recursive.

\begin{definition}[Mutually recursive atoms]\label{def:mutually-recursive-atom}
Let $\PP$ be a program and $r$ a rule in $\PP$.
The head atom $A=head(r)$ and an atom $B\in body(r)$ are \emph{mutually recursive} if there is an SCC $\calc$ of $\unig{\PP}$ s.t.: 
\begin{enumerate}
\item
$\calc$ contains $r$, and 
\item 
$\calc$ contains a rule $r'$ (possibly equal to $r$) s.t. $head(r')$ and $B$ unify.~\hfill$\Box$
\end{enumerate}
\end{definition} 

In the previous definition, when $r=r'$ we assume that $r$ and $r'$ are two ``copies'' that do not share any \lvar.
Intuitively, the head atom $A$ of a rule $r$ and an atom $B$ in the body of $r$ are mutually recursive when there might be an actual propagation of terms from $A$ to $B$ (through the application of a sequence of rules).
As a very simple example, if we have an SCC consisting only of the rule $\tt p(f(X)) \leftarrow p(X), p(g(X))$, the first body atom is mutually recursive with the head, while the second one is not as it does not unify with the head atom.

Given a rule $r$, we use $\rbody(r)$ to denote the set of atoms in $body(r)$ which are mutually recursive with $head(r)$.
Moreover, we define $\sbody(r)$ as the set consisting of every atom in $body(r)$ that contains all \lvars\ appearing in $head(r)$, and define $\srbody(r)=\rbody(r)\cap \sbody(r)$.

We say that a rule $r$ in a program $\PP$ is \emph{\relevant} if it is not a fact and the set of atoms $body(r) \setminus \rbody(r)$ does not contain all \lvars\ in $head(r)$.
Roughly speaking, a non-relevant rule will be ignored because either it cannot propagate terms or its head size is bounded by body atoms which are not mutually recursive with the head.
We illustrate the notions introduced so far in the following example.

\begin{example}\label{ex:simple_locally_bounded}
Consider the following program $\PP_{\ref{ex:simple_locally_bounded}}$:
\[
  \begin{array}{lll}
   r_1:\ \underbrace{\tt s(f(X),Y)}_\text{A} & \leftarrow & \tt \underbrace{\tt q(X,f(Y))}_\text{B},\ \underbrace{\tt s(Z,f(Y))}_\text{C}.\\
   r_2:\ \underbrace{\tt q(f(U),V)}_\text{D} & \leftarrow & \tt \underbrace{\tt s(U,f(V))}_\text{E}.\\
  \end{array}
 \]
The \ug\ consists of the edges $\<r_1,r_1\>$, $\<r_1,r_2\>$, $\<r_2,r_1\>$.
Thus, there is only one SCC $\calc=\{r_1,r_2\}$, which is non-trivial, and $pred(\calc)=\tt \{q,s\}$.
Atoms $A$ and $B$ (resp. $A$ and $C$, $D$ and $E$) are mutually recursive.
Moreover, $\rbody(r_1)=\{B,C\}$, $\srbody(r_1)=\{B\}$, $\rbody(r_2)=\srbody(r_2)=\{E\}$.
Both $r_1$ and $r_2$ are \relevant.~\hfill$\Box$
\end{example}

We use $\mathbb{N}$ to denote the set of natural numbers $\{1,2,3,\dots\}$ and $\mathbb{N}_0$ to denote the set of natural numbers including the zero.
Moreover, $\mathbb{N}^k=\{(v_1,\dots,v_k)\mid  v_i \in \mathbb{N} \mbox{ for } 1 \leq i \leq k\}$ and $\mathbb{N}_0^k=\{(v_1,\dots,v_k)\mid  v_i \in \mathbb{N}_0 \mbox{ for } 1 \leq i \leq k\}$.  
Given a $k$-vector $\overline{v}=(v_1,\dots,v_k)$ in $\mathbb{N}_0^k$, we use $\overline{v}[i]$ to refer to $v_i$, for $1 \leq i \leq k$.
Given two $k$-vectors $\overline{v}=(v_1,\dots,v_k)$ and $\overline{w} = (w_1,\dots,w_k)$ in $\mathbb{N}_0^k$, we use $\overline{v} \product \overline{w}$ to denote the classical scalar product, i.e.,
$\overline{v} \product \overline{w} = \sum_{i=1}^k v_i \cdot w_i$.

As mentioned earlier, the basic idea of the proposed technique is to measure the size of terms and atoms in order to check if the rules' head sizes are bounded when propagation occurs.
Thus, we introduce the notions of term and atom size. 

\begin{definition}\label{def:size}
Let $t$ be a term. 
The \emph{size} of $t$ is recursively defined as follows:
$$size(t)=
\begin{cases} 
x &  \mbox{ if } t \mbox{ is a \lvar\ } X ; \\ 
m + \sum\limits_{i=1}^m size(t_i) & \mbox{ if }  t = f(t_1,\dots,t_m). 
\end{cases} 
$$
\noindent
where $x$ is an \nvar.
The \emph{size} of an atom $A=p(t_1,\dots,p_n)$, 
denoted $size(A)$, is the $n$-vector $(size(t_1),\dots,size(t_n))$.~\hfill$\Box$
\end{definition}

In the definition above, an \nvar\ $x$ intuitively represents the possible sizes that the \lvar\ $X$ can have during the bottom-up evaluation.
The size of a term of the form $f(t_1,\dots,t_m)$ is defined by summing up the size of its terms $t_i$'s plus the arity $m$ of $f$.
Note that from the definition above, the size of every constant is 0.
\begin{example}\label{ex:term-size}
Consider rule $r_1$ of program $\PP_{\ref{ex:bubble}}$ (see Example~\ref{ex:bubble}).
Using $\tt lc$ to denote the list constructor operator ``$|$'', the rule can be rewritten as follows:
$$
\tt bub(lc(Y,T),lc(X,Cur),Sol)  \leftarrow \tt bub(lc(X,lc(Y,T)),Cur,Sol),X \leq Y.
$$
Let $A$ (resp. $B$) be the atom in the head (resp. the first atom in the body).
Then,
\[
\begin{array}{lll}
\hspace*{2.5cm} size(A) & = & (2+y+t,\ \ 2+x+cur,\ \  sol) \\
\hspace*{2.5cm} size(B) & = & (2+[x+(2+y+t)],\ \ cur,\ \ sol) \hspace*{2.5cm}\Box
\end{array}
\]
\end{example}

We are now ready to define \lb\ programs.

\begin{definition}[\Lb\ programs]\label{def:lb-program}
Let $\PP$ be a program, $\calc$ a non-trivial SCC of $\unig{\PP}$, and $pred(\calc)=\{p_1,\dots,p_k\}$.
We say that $\calc$ is \emph{\lb} if there exist $k$ vectors 
$\overline{\alpha}_{p_h} \in \mathbb{N}^{\ar(p_h)}$, $1 \leq h \leq k$,  such that for every \relevant\ rule $r \in \calc$ with $A=head(r)=p_i(t_1,\dots,t_n)$, there exists an atom $B=p_j(u_1,\dots,u_m)$ in $\srbody(r)$ s.t. the following inequality is satisfied
$$
\overline{\alpha}_{p_j} \product size(B) - \overline{\alpha}_{p_i} \product size(A) \ge 0
$$
for every non-negative value of the \nvars\ in $size(B)$ and $size(A)$. 

We say that $\PP$ is \emph{\lb} if every non-trivial SCC of $\unig{\PP}$ is \lb.~\hfill$\Box$
\end{definition}

Intuitively, for every relevant rule of a non-trivial SCC of $\unig{\PP}$, Definition \ref{def:lb-program} checks if the  size of the head atom is bounded by the size of a mutually recursive body atom for all possible sizes the terms can assume.

\begin{example}
Consider again program $\PP_{\ref{ex:simple_locally_bounded}}$ of Example~\ref{ex:simple_locally_bounded}.
Recall that the only non-trivial SCC of $\unig{\PP_{\ref{ex:simple_locally_bounded}}}$ is $\calc=\{r_1,r_2\}$, and both $r_1$ and $r_2$ are \relevant. 
To determine if the program is \lb\ we need to check if $\calc$ is \lb.
Thus, we need to find $\overline{\alpha}_q, \overline{\alpha}_s \in \mathbb{N}^2$ such that there is an atom in $\srbody(r_1)$ and an atom in $\srbody(r_2)$ which satisfy the two inequalities derived from $r_1$ and $r_2$ for all non-negative values of the \nvars\ therein.
Since both $\srbody(r_1)$ and $\srbody(r_2)$ contain only one element, we have only one choice, namely the one where $B$ is selected for $r_1$ and $E$ is selected for $r_2$.

\noindent Thus, we need to check if there exist $\overline{\alpha}_q, \overline{\alpha}_s \in \mathbb{N}^2$ s.t. the following linear constraints are satisfied for all non-negative values of the \nvars\ appearing in them:
\[
  \begin{array}{lll}
 \!
 \begin{cases}
    \overline{\alpha}_q \product size(B) - \overline{\alpha}_s \product size(A)  \ge 0 \\
    \overline{\alpha}_s \product size(E) - \overline{\alpha}_q \product size(D)  \ge 0
 \end{cases}
 & \! \! \! \! \! \!  \Rightarrow \! &
  \begin{cases}
    \overline{\alpha}_q \product (x,1+y) - \overline{\alpha}_s \product (1+x,y)  \ge 0 \\
    \overline{\alpha}_s \product (u,1+v) - \overline{\alpha}_q \product (1+u,v)  \ge 0
 \end{cases}
 \end{array}
\]
By expanding the scalar products and isolating every \nvar\ we obtain:
\[
 \begin{cases}
(\overline{\alpha}_q[1] - \overline{\alpha}_s[1]) \product x + (\overline{\alpha}_q[2] - \overline{\alpha}_s[2]) \product y + (\overline{\alpha}_q[2] - \overline{\alpha}_s[1])  \geq 0 \\
  (\overline{\alpha}_s[1] - \overline{\alpha}_q[1]) \product u + (\overline{\alpha}_s[2] - \overline{\alpha}_q[2]) \product v + (\overline{\alpha}_s[2] - \overline{\alpha}_q[1])  \geq 0
 \end{cases}
 \]
The previous inequalities must hold for all $x,y,u,v \in \mathbb{N}_0$; it is easy to see that this is the case iff the following system admits a solution:
\[\begin{cases}
 \begin{array}{lll}
  \overline{\alpha}_q[1] - \overline{\alpha}_s[1] \ge 0, & \ \ \overline{\alpha}_q[2] - \overline{\alpha}_s[2] \ge 0, & \ \ \overline{\alpha}_q[2] - \overline{\alpha}_s[1] \geq 0, \\
  \overline{\alpha}_s[1] - \overline{\alpha}_q[1] \ge 0, & \ \ \overline{\alpha}_s[2] - \overline{\alpha}_q[2] \ge 0, & \ \ \overline{\alpha}_s[2] - \overline{\alpha}_q[1]  \geq 0
 \end{array}
 \end{cases}
 \]
Since a solution does exist, e.g. $\overline{\alpha}_s[1]=\overline{\alpha}_s[2]=\overline{\alpha}_q[1]=\overline{\alpha}_q[2]=1$ (recall that every $\overline{\alpha}[i]$ must be greater than 0), the SCC $\calc$ is \lb, and so is the program.~\hfill$\Box$
\end{example}

The method in the previous example  to find vectors $\overline{\alpha}_p$ for all $p \in pred(\calc)$ can~always be applied. 
That is, we can always isolate the \nvars\ in the original inequalities and then derive one inequality for each expression that multiplies  an \nvar\ plus the one for the constant term, imposing that all such expressions must be greater than or equal to 0---we precisely state this property in Lemma~\ref{lem:constraint}. 

It is worth noting that the proposed technique can easily recognize many terminating practical programs where terms are simply exchanged from the body to the head of rules (e.g., see Examples~\ref{ex:bubble},~\ref{ex:visit},~and~\ref{ex:append}).

\begin{example}\label{ex:bubble_continued}
Consider program $\PP_{\ref{ex:bubble}}$ of Example~\ref{ex:bubble}. 
Recall that the only non-trivial SCC of $\unig{\PP_{\ref{ex:bubble}}}$ is $\{r_1,r_2,r_3\}$ (see Figure~\ref{fig:activation}) and all rules in it are \relevant.
Since $|\srbody(r_i)|=1$ for every $r_i$ in the SCC, we have only one set of inequalities, which is the following one after isolating \nvars (we assume that
the empty list is represented by a simple constant): 
$$\begin{cases}
 (\overline{\alpha}_{b}[1] - \overline{\alpha}_{b}[2]) \product x_1 + (2\overline{\alpha}_{b}[1] -2\overline{\alpha}_{b}[2]) \ge 0 \\
 (\overline{\alpha}_{b}[1] - \overline{\alpha}_{b}[2]) \product y_2 + (2\overline{\alpha}_{b}[1] -2\overline{\alpha}_{b}[2]) \ge 0 \\
 (\overline{\alpha}_{b}[1] - \overline{\alpha}_{b}[3]) \product x_3 + (\overline{\alpha}_{b}[2] - \overline{\alpha}_{b}[1])\product cur_3 + (2\overline{\alpha}_{b}[1] - 2\overline{\alpha}_{b}[3]) \ge 0
\end{cases}
$$
where subscript $b$ stands for predicate symbol $\tt bub$, whereas subscripts associated with integer variables
are used to refer to the occurrences of logical variables in different rules (e.g., $y_2$ is the integer variable
associated to the logical variable $\tt Y$ in rule $r_2$).
A possible solution is $\overline{\alpha}_{{}_{b}}=(1,1,1)$ and thus $\PP_{\ref{ex:bubble}}$ is \lb.

Considering program $\PP_{\ref{ex:visit}}$ of Example~\ref{ex:visit}, we obtain the following constraints:
$$
\begin{cases}
(\overline{\alpha}_{v}[1] - \overline{\alpha}_{v}[2])\product root_1 + (\overline{\alpha}_{v}[1] - \overline{\alpha}_{v}[3])\product right_1 + (3\overline{\alpha}_{v}[1] -2\overline{\alpha}_{v}[2] - 
2\overline{\alpha}_{v}[3])  \ge 0 \\
(\overline{\alpha}_{v}[3] - \overline{\alpha}_{v}[1]) \product next_2 + 2\overline{\alpha}_{v}[3]  \ge 0
\end{cases}
$$
where subscript $v$ stands for predicate symbol $\tt visit$.
By setting $\overline{\alpha}_{v} = (2,1,2)$, we get positive integer values of $\overline{\alpha}_{v}[1],\overline{\alpha}_{v}[2],\overline{\alpha}_{v}[3]$ s.t. the inequalities above are satisfied for all $root_1,$ $right_1,$ $next_2 \in\mathbb{N}_0$.
Thus, $\PP_{\ref{ex:visit}}$ is \lb.

The \ug\ of program $\PP_{\ref{ex:append}}$ of Example~\ref{ex:append} has two non-trivial SCCs $\calc_1=\{r_1\}$ and $\calc_2=\{r_3\}$.
The constraints for $\calc_1$ are:
\begin{center}
$
\begin{cases}
 (\overline{\alpha}_{r}[1] -\overline{\alpha}_{r}[2] ) \product x_1 + (2\overline{\alpha}_{r}[1] - 2\overline{\alpha}_{r}[2]) \ge 0 
\end{cases}
$
\end{center}
\noindent
where subscript $r$ stands for predicate symbol $\tt reverse$.
It is easy to see that by choosing any (positive integer) values of $\overline{\alpha}_{r}[1]$ and $\overline{\alpha}_{r}[2]$ such that $\overline{\alpha}_{r}[1] \ge \overline{\alpha}_{r}[2]$, the inequality above holds for all $x_1 \in \mathbb{N}_0$. 
Likewise, the constraints for $\calc_2$ are
\begin{center}
$
\begin{cases}
 (\overline{\alpha}_{a}[1] -\overline{\alpha}_{a}[2] ) \product x_3 + (2\overline{\alpha}_{a}[1] - 2\overline{\alpha}_{a}[2]) \ge 0 
\end{cases}
$
\end{center}
where subscript $a$ stands for predicate symbol $\tt append$.
By choosing any (positive integer) values of $\overline{\alpha}_{a}[1]$ and $\overline{\alpha}_{a}[2]$ such that $\overline{\alpha}_{a}[1] \ge \overline{\alpha}_{a}[2]$, the inequality above holds for all $x_3 \in \mathbb{N}_0$. 
Thus, $\PP_{\ref{ex:append}}$ is \lb.~\hfill$\Box$
\end{example}

\section{Correctness and expressiveness}\label{sec:correct-express-lb}

In this section, we show that every \lb\ program is terminating and provide results on the relative expressiveness of \lb\ programs and other criteria. 

Note that every program $\PP$ can be partitioned into an ordered sequence of sub-programs $\PP_1,\dots,\PP_n$, called \emph{stratification}, such that, for every $1 \leq i \leq n$, every rule $r$ in $\PP_i$
depends only on rules belonging to some sub-program $\PP_j$ with $1 \leq j \leq i$.
Recall that a rule $r$ depends on a rule $r'$ if $r$ can be reached from $r'$ through the edges of the \ug.
Moreover, there always exists a stratification where every sub-program $\PP_i$ is either a non-trivial SCC or a set of trivial SCCs.
Given a set of facts $D$, it is well known that $\MM(\PP \cup D)$ can be defined in terms of the minimal model of the $\PP_i$'s following the order of the partition as follows: if $M_0 = D$ and $M_{i} = \MM(\PP_i \cup M_{i-1})$ for $1 \leq i \leq n$, then $M_n=\MM(\PP \cup D)$.

\begin{lemma}\label{lem:scc-termination}
A program $\PP$ is terminating iff every non-trivial SCC of $\unig{\PP}$ is terminating.
\end{lemma}
\begin{proof}
($\Rightarrow$)
Clearly, if there is an SCC which is not terminating, then $\PP$ is not terminating. 

\noindent
($\Leftarrow$)
Assume now that $\PP$ does not terminate and all its non-trivial SCCs terminates. 
This means that there is a set of facts $D$ such that 
the fixpoint of $\PP \cup D$ is not finite. 
Since $\PP \cup D$ can be partitioned into $(\PP_1,\dots,\PP_n)$, there must be a non-trivial (i.e. recursive) SCC $\PP_i$
such that $\PP_i \cup M_{i-1}$ does not terminate.
This contradicts the hypothesis that all non-trivial SCCs terminate.
Indeed if $P_i$ terminates, then for every set of facts $D'$ including the facts
in $M_{i-1}$, the fixpoint of $\PP_i \cup D'$ terminates and, therefore,
the fixpoint of $\PP_i \cup M_{i-1}$ terminates as well.\hfill
\end{proof}

We now refine the previous lemma by showing that to see if a program $\PP$ is terminating it is not necessary to analyze every non-trivial SCC entirely, but we can focus on its relevant rules.
Henceforth, for every set of rules $\calc$, we use $Rel(\calc)$ to denote the set of relevant rules of $\calc$.

\begin{lemma}\label{lem:relevant-termination}
Let $\PP$ be a program and let $\calc$ be an SCC of $\unig{\PP}$. 
Then, $\calc$ is terminating iff $Rel(\calc)$ is terminating.
\end{lemma}

\begin{proof}
It follows from the fact that we can derive only a finite number of ground atoms using the rules in $ground(\calc) \setminus ground(Rel(\calc))$ starting from a finite set of facts---recall that, by definition, every non-relevant rule has a set of atoms in the body that are not mutually recursive with the head and contain all variables in the head.~\hfill
\end{proof}

To show the correctness of our approach, we first show that every \lb\ program can be rewritten into an ``equivalent'' program belonging to a simpler class of programs, called \emph{\sb}.
Then, we prove that \sb\ programs are terminating and this entails that \lb\ programs are terminating as well.

\begin{definition}[Program expansion]\label{def:expansion}
Let $\PP$ be a program and let $\omega = \{ \overline{\omega}_{p_1},\dots,\overline{\omega}_{p_n}\}$ be a set of vectors
such that $\overline{\omega}_{p_i} \in \mathbb{N}^{\ar(p_i)}$ and $p_i \in pred(\PP)$ for $1 \leq i \leq n$.
For any atom $A = p(t_1,\dots,t_m)$ occurring in $\PP$, we define $\ex{A}{\omega}=A$, if $p \not \in pred(\PP)$, otherwise
$\ex{A}{\omega}=p(\overline{t}_1,\dots,\overline{t}_m)$, where each
$\overline{t}_j$ is the sequence $t_j,\dots,t_j$ of length $\omega_p[j]$. Finally, 
$\PP^\omega$ denotes the program derived from $\PP$ by replacing every atom $A$ with $\ex{A}{\omega}$.
\hfill $\Box$
\end{definition}

Intuitively, the expansion of a program is obtained from the original program by increasing the arity of each predicate symbol according to $\omega$.
Below is an example.

\begin{example}\label{ex:program-expansion}
Consider program $P_{\ref{ex:simple_locally_bounded}}$ of Example~\ref{ex:simple_locally_bounded}
and the set of vectors $\omega = \{ \overline{\omega}_{\tt s}, \overline{\omega}_{\tt q}\}$ where 
$\overline{\omega}_{\tt s} = (2,3)$ and $\overline{\omega}_{\tt q} = (2,1)$.
The program $\ex{P_{\ref{ex:simple_locally_bounded}}}{\omega}$ is as follows:
\[
  \begin{array}{lll}
\hspace*{10mm}   r_1:\tt\ s(f(X),f(X),Y,Y,Y) & \leftarrow & \tt q(X,X,f(Y)),\ \ s(Z,Z,f(Y),f(Y),f(Y)).\\
\hspace*{10mm}   r_2:\tt\ q(f(U),f(U),V)     & \leftarrow & \tt s(U,U,f(V),f(V),f(V)). \hspace*{30mm} \Box
  \end{array}	
 \]
\end{example}

We now show that for every program $\PP$ and every set of vectors $\omega$, $\PP$ is terminating iff $\ex{\PP}{\omega}$ is terminating.
In the following, for every program $\PP$, we define $\omega(\PP)=\{\ \{\overline{\omega}_{p_1},\dots,\overline{\omega}_{p_n}\}\mid p_i \in pred(\PP) \wedge \overline{\omega}_{p_i} \in \mathbb{N}^{\ar(p_i)}\}$.

\begin{lemma}\label{lem:expansion-termination}
For every program $\PP$ and every $\omega \in \omega(\PP)$, $\PP$ is terminating iff $\ex{\PP}{\omega}$ is terminating.
\end{lemma}
\begin{proof}
For every atom $A^\omega$ occurring in $\ex{\PP}{\omega}$ let $A$ be the corresponding atom in $\PP$.
The claim follows from the observation that whenever there is a instance $D$ such that 
$\lfp{\PP \cup D}^\infty(\emptyset)$ is infinite, it is always possible to construct the instance 
$\ex{D}{\omega}$ which guarantees that $\lfp{\ex{\PP}{\omega} \cup \ex{D}{\omega}}^\infty(\emptyset)$ is infinite
as well. \\
Conversely, for every instance $D^\omega$ of $\ex{\PP}{\omega}$, 
if $\lfp{\ex{\PP}{\omega} \cup D^\omega}^\infty(\emptyset)$
is infinite, then we can always construct the instance $D$ guaranteeing that $\lfp{\PP \cup D}^\infty(\emptyset)$ 
is infinite as well. \hfill 
\end{proof}

We now introduce the class of \sb\ programs and show that such programs are terminating.
To this aim, we define the \emph{\tsize} of an atom $A=p(t_1,\dots,t_n)$ as $\ts(A)=\sum\limits_{i=1}^n size(t_i)$.

\begin{definition}[\Sb\ program]
A program $\PP$ is said to be \emph{\sb} if for every rule $r \in \PP$ which is not a fact,
there is an atom $B$ in $\sbody(r)$ such that $\ts(B) \ge \ts(head(r))$ for every
non-negative value of the \nvars\ occurring in $\ts(B)$ and $\ts(head(r))$.
\end{definition}

\begin{theorem}\label{the:size-termination}
Every \sb\ program is terminating.
\end{theorem}
\begin{proof}
Let $\PP$ be a \sb\ program and $D$ a finite set of facts, we consider only rules in $\PP$ having a non-empty body.
Given an atom $A$ and a ground instance $A'$ of $A$, let $\theta$ be the mgu of $A$ and $A'$.
Notice that $\theta$ is of the form $\{X_1/t_1, \dots, X_n/t_n\}$ where the $X_i$'s are exactly the \lvars\ occurring in $A$ and all the $t_j$'s are ground terms.
It can be easily verified that $\ts(A')$ can be obtained from $\ts(A)$ by replacing every \nvar\ $x_i$ in $\ts(A)$ with $size(t_i)$.

We now show that for every ground rule $r' \in ground(\PP)$ there is an atom $B' \in body(r')$ such that $\ts(B') \ge \ts(head(r'))$.
Consider a rule $r$ in $\PP$ of the form $A \leftarrow B_1,\dots, B_k$ and a ground rule $r' \in ground(r)$ of the form $A' \leftarrow B_1',\dots, B_k'$.
Since $\PP$ is \sb, there exists an atom $B_j$ in $\sbody(r)$ such that $\ts(B_j) \ge \ts(A)$ for every non-negative value of the \nvars\ occurring in the inequality.
Notice every \lvar\ in $A$ appears also in $B_j$ by definition of $\sbody$.
Let $\{X_1/t_1, \dots, X_n/t_n\}$ be the mgu of $B_j$ and  $B_j'$.
As $\ts(B_j) \ge \ts(A)$ holds for all non-negative value of its \nvars, it also holds when every \nvar\ $x_i$ is replaced with $size(t_i)$, for $1 \leq i \leq n$.
Thus, $\ts(B_j') \ge \ts(A')$.

Let us denote $\lfp{\PP \cup D}^i(\emptyset)$ as $M_i$ for every $i\geq 1$ and
let $\ts_{max} = \max\{\ts(B) \mid  B \leftarrow \mbox{ is a fact in } \PP \cup D\}$.
We show that for every $i\geq 1$ and every ground atom $A$ in $M_i$ the following holds
$\ts_{max} \ge \ts(A)$.
The proof is by induction on $i$.

\noindent 
$\bullet$
\emph{Base case ($i$=$1$).} It follows from the fact that $M_1$=$\{B \mid  B\!\leftarrow\!\mbox{ is a fact in } \PP\,\cup\,D\}$.

\noindent 
$\bullet$
\emph{Inductive step ($i\rightarrow i+1$).} 
Let $r'$ be a ground rule in $ground(\PP)$ such that $body(r')\subseteq M_i$.
Then, as shown above, there is an atom $B$ in $body(r')$ such that $\ts(B) \ge \ts(head(r'))$.
By the induction hypothesis, $\ts_{max} \ge \ts(B)$ and thus $\ts_{max} \ge \ts(head(r'))$.

Thus, for every $i\geq 1$ and every ground atom $A$ in $M_i$, we have that $\ts(A)$ is bounded by $\ts_{max}$.
Since programs are range-restricted, atoms in $\cup_{i\geq 1}M_i$ are built from constants and function symbols appearing in $\PP\cup D$, which are finitely many.
These observations and the definition of $\ts$ imply that we can have only finitely many ground atoms in $\cup_{i\geq 1}M_i$.
Hence, $\PP$ is terminating.\hfill
\end{proof}

We are now ready to show the correctness of the \lb\ technique.

\begin{theorem}\label{the:lb-equivalence}
Every \lb\ program is terminating.
\end{theorem}
\begin{proof}
Let $\PP$ be a \lb\ program and $\calc$ a non-trivial SCC of $\unig{\PP}$.
Since $\PP$ is \lb, then there exists $\omega \in \omega(\calc)$ which satisfies the condition of Definition~\ref{def:lb-program}, that is, $\calc$ is \lb.
This implies that $Rel(\calc)^\omega$ is \sb.
Thus, $Rel(\calc)^\omega$ is terminating by Theorem~\ref{the:size-termination}.
Lemma~\ref{lem:expansion-termination} implies that $Rel(\calc)$ is terminating and Lemma~\ref{lem:relevant-termination} in turn implies that $\calc$ is terminating.
Finally, by Lemma~\ref{lem:scc-termination}, we can conclude that $\PP$ is terminating.\hfill
\end{proof}

The class of \lb\ programs is incomparable with different termination criteria in the literature, including the most general ones.

\begin{theorem}\label{th:comparison-1}
\Lb\ programs are incomparable with argument-restricted, mapping-re-stricted, and bounded programs.
\end{theorem}
\begin{proof}
Recall that both bounded and mapping-restricted programs include argument-restricted programs. 
To prove the claim we show that 
\emph{(i)} there is a program which is \lb\ but is neither mapping-restricted nor bounded, and 
\emph{(ii)} there is a program which is argument-restricted but not \lb. 
\emph{(i)} As already shown, program $\PP_{\ref{ex:bubble}}$ of Example~\ref{ex:bubble} is \lb; however, it can be easily verified that $\PP_{\ref{ex:bubble}}$ is neither mapping-restricted nor bounded.
\emph{(ii)} Consider the program consisting of the rules $\tt p(f(X)) \leftarrow q(X)$ and $\tt q(Y) \leftarrow p(f(Y))$. 
This program is argument-restricted (and thus also mapping-restricted and bounded) but is not \lb.~\hfill
\end{proof}

Regarding the termination criteria mentioned in Theorem~\ref{th:comparison-1},
we recall that mapping restriction $(MR)$ and bounded programs $(BP)$ are incomparable and both extend argument restriction $(AR)$.
Concerning the computational complexity, while $AR$ is polynomial time, both $MR$ and $BP$ are exponential. 
As a remark, it is interesting to note that the above result highlights the fact that our technique 
analyzes logic programs from a radically different point of view \wrt previously defined approaches, which analyze 
how complex terms are propagated among arguments.

\section{\GB\ Programs}\label{sec:mutual_recursion}

As saw in the previous section, to determine if a program is \lb\ we check through linear constraints if the size of the head atom is bounded by the size of a body atom for every relevant rule in a non-trivial SCC of the \ug\ (cf. Definition~\ref{def:lb-program}).
Looking at each rule individually has its limitations, as shown by the following example.

\begin{example}\label{ex:globally-bounded}
Consider the following simple program $\PP_{\ref{ex:globally-bounded}}$:
\[
\begin{array}{lll}
r_1:\ \tt p(X,Y) & \leftarrow & \tt q(f(X),Y).\\
r_2:\ \tt q(W,f(Z)) & \leftarrow & \tt p(W,Z). 
\end{array}
\]
It is easy to see that the program above is terminating, but it is not \lb.
The linear inequalities for the program are (cf. Definition~\ref{def:lb-program}):
$$
\begin{cases}
(\overline{\alpha}_{q}[1] -\overline{\alpha}_{p}[1]) \product x + (\overline{\alpha}_{q}[2] - \overline{\alpha}_{p}[2]) \product y + \overline{\alpha}_{q}[1] \ge 0\\
(\overline{\alpha}_{p}[1] -\overline{\alpha}_{q}[1]) \product w + (\overline{\alpha}_{p}[2] - \overline{\alpha}_{q}[2]) \product z -\overline{\alpha}_{q}[2] \ge 0
\end{cases}
$$
\noindent
It can be easily verified that there are no positive integer values for 
$\overline{\alpha}_{p}[1],$ $\overline{\alpha}_{p}[2],$ 
$\overline{\alpha}_{q}[1],$ $\overline{\alpha}_{q}[2]$ 
such that the inequalities hold for all $x,y,w,z \in \mathbb{N}_0$.
The reason is the presence of the expression $-\overline{\alpha}_{q}[2]$ in the second inequality. 
Intuitively, this is because the size of the head atom increases \wrt the size of the body atom in $r_2$.
However, notice that the cycle involving $r_1$ and $r_2$ does not increase the overall size of propagated terms.
This suggests we can check if an \emph{entire cycle} (rather than each individual rule) propagates terms of bounded size.~\hfill$\Box$
\end{example}

To deal with programs like the one shown in the previous example, we introduce the class of \emph{\gb\ programs}, which is able to perform an analysis of how terms propagate through a \emph{group} of rules, rather than looking at rules \emph{individually} as done by the \lb\ criterion. 

Given a program $\PP$, a cyclic path $\pi$ of $\unig{\PP}$ is a sequence of edges 
$\<r_1,r_2\>,\<r_2,r_3\>,$ $\dots,\<r_n,r_1\>$. 
Moreover a cyclic path $\pi$ is \emph{basic} if every edge $\pi$ does not occur more than once.
We say that $\pi$ is \emph{relevant} if every $r_i$ is \relevant, for $1 \leq i \leq n$.

In the following, we first present the \gb\ criterion for linear programs and then show how it can be applied to non-linear ones.

\paragraph{\bf Dealing with linear programs.}
A program $\PP$ is \emph{linear} if every rule in $\PP$ is linear. A rule $r$ is \emph{linear} if $|\rbody(r)|\leq 1$.
Notice that $\rbody(r)$ contains exactly one atom $B$ for every linear rule $r$ in a non-trivial SCC of the \ug; thus, with a slight abuse of notation, we use $\rbody(r)$ to refer to $B$.

\begin{definition}[Cycle constraints]\label{def:linear-gb}
Let $\PP$ be a linear program and let $\pi = \<r_1,r_2\>,\dots,\<r_n,r_1\>$ be a basic cyclic path of $\unig{\PP}$.
For every mgu $\theta_i$ of $head(r_i)$ and $\rbody(r_{i+1})$ ($1 \le i < n$)\footnote{Note that such $\theta_i$'s always exist by definition of \ug.},
we define the set of (linear) equalities $eq(\theta_i)=\{x=size(t)\mid  X/t \in \theta_i\}$.
Then, we define $eq(\pi)=\bigcup\limits_{1 \le i < n} eq(\theta_i)$.
~\hfill$\Box$
\end{definition}

\begin{example}
Consider the program $\PP_{\ref{ex:globally-bounded}}$ and the two basic cyclic paths $\pi_1 = \<r_1,r_2\>\,\<r_2,r_1\>$ and
$\pi_2 = \<r_2,r_1\>\,\<r_1,r_2\>$ of $\unig{\PP_{\ref{ex:globally-bounded}}}$. The mgu of $head(r_1)$ and $rbody(r_2)$ is $\theta=\{\tt X/W,\ Y/Z\}$ and thus 
$eq(\pi_1)=\{x=w,\ y=z\}$. Furthermore, the mgu of $head(r_2)$ and $rbody(r_1)$ is $\theta=\{\tt W/f(X),\ Y/f(Z)\}$ and thus
$eq(\pi_2)=\{w=1+x,\ y=1+z\}$.~\hfill$\Box$
\end{example}

\begin{definition}[Linear \gb\ programs]\label{def:linear-gb}
Let $\PP$ be a linear program, $\pi=\<r_1,r_2\>\,\dots\,\<r_n,r_1\>$ be a basic cyclic path of $\unig{\PP}$ and $p$ be the predicate defined by $r_n$.
We say that $\pi$ is \emph{\gb} if $eq(\pi)$ is satisfiable for some non-negative value of its \nvars\ and there exists a vector $\overline{\alpha}_p \in \mathbb{N}^{\ar(p)}$ such that the constraint
$$
\overline{\alpha}_{p} \product size(\rbody(r_1)) - \overline{\alpha}_{p} \product size(head(r_n)) \ge 0
$$
is satisfied for every non-negative value of its \nvars\ that satisfy $eq(\pi)$.
We say that $\PP$ is \emph{\gb} if every relevant basic cyclic path of $\unig{\PP}$ is \gb.~\hfill$\Box$
\end{definition}

In the definition above, we assume that distinct basic cyclic paths do not share any \lvar .

\begin{example}\label{ex:globally-bounded-contd}
Consider again program~$\PP_{\ref{ex:globally-bounded}}$ of Example~\ref{ex:globally-bounded}.
The program is clearly linear and $\unig{\PP_{\ref{ex:globally-bounded}}}$ has only two relevant basic cyclic paths $\pi_1=\<r_1,r_2\>\<r_2,r_1\>$ and $\pi_2=\<r_2,r_1\>\<r_1,r_2\>$.
To check if $\PP_{\ref{ex:globally-bounded}}$ is \gb, we need to check if $eq(\pi_1)=\{ x_1=w_1,\ y_1=z_1\}$ and 
$eq(\pi_2)=\{w_2=1+x_2,\ y_2=1+z_2\}$ admit a solution
and if there exist $\overline{\alpha}_{p},\overline{\alpha}_q \in\mathbb{N}^2$ s.t. the constraints:
\[
\begin{array}{llll}
 \overline{\alpha}_{q} \product (x_1 + 1, y_1) & - & \overline{\alpha}_{q} \product (w_1,z_1+1) & \geq 0, \\
 \overline{\alpha}_{p} \product (w_2,z_2) & - & \overline{\alpha}_{p} \product (x_2,y_2) & \ge 0
 \end{array}
\]
are satisfied for all $x_1,y_1,w_1,z_1 \in\mathbb{N}_0$ and all $x_2,y_2,w_2,z_2 \in\mathbb{N}_0$ that satisfy $eq(\pi_1)$ and $eq(\pi_2)$.

By applying the equality conditions $eq(\pi_1)$ and $eq(\pi_2)$ to the above constraints
we get the below inequalities for the basic cyclic paths $\pi_1$ and $\pi_2$:
\[
\begin{array}{llll}
 (\overline{\alpha}_{q}[1], \overline{\alpha}_{q}[2])  \product (x_1 + 1, z_1) & - &  
 (\overline{\alpha}_{q}[1], \overline{\alpha}_{q}[2])) \product (x_1,z_1+1) & \geq 0, \\
 (\overline{\alpha}_{p}[1], \overline{\alpha}_{p}[2])) \product (x_2+1,z_2)   & - &
 (\overline{\alpha}_{p}[1], \overline{\alpha}_{p}[2])) \product (x_2,z_2+1) & \ge 0
 \end{array}
\]

It is easy to see that the first constraint (resp. the second) is satisfied for every vector
$\overline{\alpha}_p\in\mathbb{N}^2$ (resp. $\overline{\alpha}_q \in \mathbb{N}^2$) such that
$\overline{\alpha}_p[1] \ge \overline{\alpha}_p[2]$ (resp. $\overline{\alpha}_q[1] \ge \overline{\alpha}_q[2]$). Thus, $\PP_{\ref{ex:globally-bounded}}$ is \gb .
~\hfill$\Box$
\end{example}

To prove the correctness of our approach, we introduce a simpler class of terminating programs, as we did in the case of \lb\ programs.

\begin{definition}[Linear \csb\ programs]\label{def:linear-size-gb}
Let $\PP$ be a linear program.
We say that $\PP$ is \emph{\csb} if for every relevant basic cyclic path $\pi=\<r_1,r_2\>\,\dots\,\<r_n,r_1\>$ of $\unig{\PP}$, $eq(\pi)$ is satisfiable for some non-negative value of its \nvars\ and the constraint
$$
\ts(\rbody(r_1)) - \ts(head(r_n)) \ge 0
$$
is satisfied for every non-negative value of its \nvars\ that satisfy $eq(\pi)$.~\hfill$\Box$
\end{definition}

\begin{theorem}\label{the:size-cycle-bounded-termination}
Every linear \csb\ program is terminating.
\end{theorem}
\begin{proof}
Let $\PP$ be a \csb\ program and $D$ a finite set of facts. 
Consider a relevant basic cyclic path $\pi=\<r_1,r_2\>\,\dots\,\<r_n,r_1\>$ of $\unig{\PP}$. 
Let $r_1',\dots,r_n'$ be ground rules s.t. $r_i' \in ground(r_i)$ for $1 \leq i \leq n$ and $head(r_i')=\rbody(r_{i+1}')$ for $1 \leq i < n$.
For $1 \leq i \leq n$, let $\theta^h_i$ be the mgu of $head(r_i)$ and $head(r_i')$,
and $\theta^b_i$ the mgu of $\rbody(r_i)$ and $\rbody(r_i')$.
Then,

\noindent 
$\bullet$ $\ts(head(r_i'))$ can be obtained from $\ts(head(r_i))$ by replacing every \nvar\ $x$ in $\ts(head(r_i))$ with $size(t)$ provided that $X/t \in \theta^h_i$, for $1 \leq i \leq n$;

\noindent $\bullet$
$\ts(\rbody(r_i'))$ can be obtained from $\ts(\rbody(r_i))$ by replacing every \nvar\ $x$ in $\ts(\rbody(r_i))$ with $size(t)$ provided that $X/t \in \theta^b_i$, for $1 \leq i \leq n$;

\noindent $\bullet$ if we replace every \nvar\ $x$ in $eq(\pi)$ with $size(t)$ iff $X/t$ belongs to $\cup_{i=1}^n (\theta^h_i \cup \theta^b_i)$, then $eq(\pi)$ is satisfied.

The items above entail that $\ts(\rbody(r_1')) - \ts(head(r_n')) \ge 0$.
This means that we cannot derive atoms of increasing size through the cyclic application of rules and thus $\PP\cup D$ is terminating.
\hfill
\end{proof}

\begin{theorem}[Soundness]
Every linear \gb\ program is terminating.
\end{theorem}
\begin{proof}
The proof is similar to the one presented for \lb\ programs. 
Given a linear \gb\ program $\PP$, we are going to construct
an equivalent program (like $\ex{\PP}{\omega}$) to $\PP$ as follows: for every relevant basic cyclic path $\pi=\<r_1,r_2\>\,\dots\,\<r_n,r_1\>$ of $\unig{\PP}$,
let $\overline{\alpha}_p$ be the vector such that $\overline{\alpha}_{p} \product size(\rbody(r_1)) - \overline{\alpha}_{p} \product size(head(r_n)) \ge 0$.
Then, remove rules $r_1$ and $r_n$ from $\PP$ and insert the rules $head(r_1) \leftarrow \ex{\rbody(r_1)}{\overline{\alpha}_p}$ and 
$\ex{head(r_n)}{\overline{\alpha}_p} \leftarrow \rbody(r_n)$ respectively.
Finally, in order to preserve the activation of rules in the obtained program, for every pair of basic cyclic paths $\pi_1=\<r_1,r_2\>\,\dots\,\<r_n,r_1\>$, 
$\pi_2=\<s_1,s_2\>\,\dots\,\<s_m,s_1\>$, where $p$ is the predicate defined by $r_n$ and $s_n$ with arity $k$, add to $\PP$ a rule of the form 
$\ex{A}{\overline{\alpha}_p} \leftarrow \ex{A}{\overline{\beta}_p}$, where $A$ is the atom $p(X_1,\dots,X_k)$ and 
$\overline{\alpha}_p$, $\overline{\beta}_p$ are the vectors such that $\overline{\alpha}_{p} \product size(\rbody(r_1)) - \overline{\alpha}_{p} \product size(head(r_n)) \ge 0$ and $\overline{\beta}_{p} \product size(\rbody(s_1)) - \overline{\beta}_{p} \product size(head(s_m)) \ge 0$ respectively. It is not difficult
to show that the obtained program is terminating iff $\PP$ is terminating. Moreover, since $\PP$ is \gb\ the new program is consequently
\csb . From Theorem~\ref{the:size-cycle-bounded-termination}, we get that the new program is terminating and so it is $\PP$. 
\hfill
\end{proof}

\paragraph{\bf Dealing with non-linear programs.}
The application of the \gb\ criterion to arbitrary programs consists in applying the technique to a set of linear programs derived from the original one.
Given a rule $r$, the set of \emph{linear versions} of $r$ is defined as the set of rules $\linear(r)=\{head(r)\leftarrow B \mid  B \in \rbody(r)\}$.
Given a program $\PP=\{r_1,\dots,r_n\}$, the set of \emph{linear versions} of $\PP$ is defined as the set of linear programs $\linear(\PP)=\{\{r_1',\dots,r_n'\}\mid  r_i' \in \linear(r_i) \mbox{ for } 1\leq i\leq n\}$.

\begin{definition}[\Gb\ programs]\label{def:gb}
A (possibly non-linear) program $\PP$ is \emph{\gb} if every (linear) program in $\linear(\PP)$ is \gb.~\hfill$\Box$
\end{definition}

\begin{theorem}
Every \gb\ program is terminating.
\end{theorem}
\begin{proof}
Notice that every linear version $\PP' \in \linear(\PP)$ of $\PP$ is such that for every set of facts $D$, $\MM(D \cup \PP) \subseteq \MM(D \cup \PP')$. 
Thus, if every linear version of $\PP$ is \gb\, then for every set of facts $D$, $\MM(D \cup \PP)$ is finite.~\hfill
\end{proof}

\begin{theorem}[Expressivity]\label{th:comparison-2}
\Gb\ programs are incomparable with \lb, argument-restricted, mapping-restricted and bounded programs.
\end{theorem}
\begin{proof}
As shown in Example~\ref{ex:globally-bounded}, program $\PP_{\ref{ex:globally-bounded}}$ is \gb, but it can be easily verified that it is neither mapping-restricted (and thus not argument-restricted) nor \lb. Moreover, the one rule program $\{\tt p(X,Y,f(Z,W)) \leftarrow p(f(Z,Y),X,W).\}$ is \gb\ but it is not bounded. \\
Conversely, the program $\{\tt p(f(X)) \leftarrow p(f(f(X))),p(X).\}$ is \lb, argument-restricted (and thus mapping-restricted) and bounded but not \gb. 
\end{proof}

\section{Complexity}\label{sec:complexity}

In this section, we provide upper bounds for the time complexity of checking whether a program is \lb\ or \gb.
We assume that constant space is used to store each constant, \lvar, function symbol, and predicate symbol.
The \emph{syntactic size}\footnote{We use the name syntactic size to distinguish it from the notion of size introduced in Definition~\ref{def:size}.} of a term $t$ (resp. atom, rule, program), denoted by $\ssize{t}$, is the number of symbols
occurring in $t$, except for the symbols ``('', ``)'', ``,'', ``.'', and ``$\leftarrow$''. 
Thus, in this section, the complexity of a problem involving $\PP$ is assumed to be w.r.t. $\ssize{\PP}$.
Obviously $|\PP| = O(\ssize{\PP})$.

\begin{lemma}\label{lem:firing-graph-complexity}
Given a program $\PP$, constructing $\unig{\PP}$ is in $\PTIME$.
\end{lemma}
\begin{proof}
The construction of $\unig{\PP}$ requires checking, for every atom $A$ in the head of a rule and every atom $B$ in the body of a rule, if $A$ and $B$ unify.
Since we need to check $|\PP| \times \sum_{r \in \PP} |body(r)|$ times if two atoms unify and checking whether two atoms $A$ and $B$ unify can be done in quadratic time w.r.t. $\ssize{A}$ and $\ssize{B}$ \cite{Venturini75}, then the construction of $\unig{\PP}$ is in $\PTIME$.\hfill
\end{proof}

It is worth noting that the number of SCCs is bounded by $O(|\PP|)$ and that 
after having built $\unig{\PP}$, the cost of checking whether a SSC is trivial or nontrivial 
is constant, whereas the cost of checking whether a rule is relevant is bounded by $O(\ssize{\PP})$.
Inequalities associated with basic cycles can be rewritten by grouping terms
with respect to integer coefficients (also called $\alpha$-coefficients) or with respect to integer variables.
Therefore, in the following we assume that inequalities grouped with respect to integer
variables are of the form
$
\gamma_1 \product x_1, + \cdots + \gamma_n \product x_n + \gamma_0 \geq 0,
$
where each $\gamma_i$, for $0 \leq i \leq n$, is an arithmetic expression built
by using $\alpha$-coefficients and natural numbers, whereas
inequalities grouped with respect to integer coefficients are of the form
$
\alpha_1 \product w_1, + \cdots + \alpha_m \product w_m \geq 0,
$
where each $w_j$, for $1 \leq i \leq m$, is an arithmetic expression built
by using integer variables and natural numbers.
Obviously, each $\gamma_i$ can be considered an integer coefficient,
whereas each $w_j$ can be considered an integer variable.

\begin{lemma}\label{lem:constraint}
Consider a linear inequality of the form
$$ \gamma_1 \product x_1 + \dots + \gamma_n \product x_n + \gamma_{0} \ge 0$$
where the $\gamma_i$'s are integer coefficients and the $x_j$'s are \nvars.
The inequality is satisfied for every non-negative value of the $x_j$'s iff $\gamma_i \geq 0$ for every $0 \leq i \leq n$.
\end{lemma} 
\begin{proof}
($\Leftarrow$) Straightforward.
\noindent
($\Rightarrow$) By contradiction, assume that the inequality is satisfied for every non-negative value of the \nvars\ occurring in it, but there exists $0 \leq i \leq n$ such that $\gamma_i<0$. 
If $1 \leq i \leq n$, then the inequality is not satisfied when $x_i=\lfloor abs(\gamma_{n+1}/\gamma_i) \rfloor +1$ and $x_j=0$ for every $j \neq i$.
If $i=0$, then the inequality is not satisfied when $x_j=0$ for every $1 \leq j \leq n$.\hfill
\end{proof}

\begin{theorem}\label{the:rule-bounded-complexity}
Checking whether a program $\PP$ is \lb\ is in $\NP$.
\end{theorem}
\begin{proof}
In order to check whether $\PP$ is \lb\ we need to:
1) construct the \ug\ $\unig{\PP}$ of $\PP$,
2) compute the SCCs of $\unig{\PP}$, and 
3) check if every non-trivial SCC is \lb.

\noindent
1) The construction of the \ug\ is in $\PTIME$ by Lemma~\ref{lem:firing-graph-complexity}. \\
2) It is well known that computing the SCCs of a directed graph can be done in linear time w.r.t. the number of nodes and edges.
Since the number of nodes of $\unig{\PP}$ is $|\PP|$ and the maximum number of edges of  $\unig{\PP}$ is $|\PP|^2$, then computing all the SCCs is clearly in $\PTIME$. \\
3) Let $\calc$ be a non-trivial SCC of $\unig{\PP}$, $n = O(|\PP|)$ the number of \relevant\ rules in $\calc$, 
$v$ the maximum number of distinct variables occurring in the head atoms of the \relevant\ rules in $\calc$, 
and $a$ the maximum arity of the predicate symbols in $pred(\calc)$.
Since it is always possible to rewrite the constraints as in Definition~\ref{def:lb-program} in the form 
presented by Lemma~\ref{lem:constraint}, given a fixed choice of one atom in $\srbody(r)$ 
for every relevant rule $r$ of $\calc$, checking whether $\calc$ is \lb\ according to that choice can be 
done by solving a set of at most $n \times (v+1)$ linear constraints with at most $2 \times a$ 
non-negative coefficients per constraint---clearly, the size of the set of constraints is bounded by $O(\ssize{\PP})$
and if the set of constraints admit a solution, then
there is a solution where the size of the $\alpha$-coefficients is polynomial in the size of $\ssize{\PP}$  
(bounded by $O(v \times n \times k)$, where $k$ is the maximum constant appearing in the set of inequalities).
As checking if such a set of linear constraints admits a solution can be done in non-deterministic polynomial time \cite{Papadimitriou81}, it follows from the above discussion that this can be checked in polynomial time.

Hence, checking whether $\PP$ is \lb\ is in $\NP$.\hfill
\end{proof}

We discuss now the complexity of checking whether a program is cycle-bounded.
To this aim, we first introduce a technical lemma similar to Lemma \ref{lem:constraint}. 

\begin{lemma}\label{lem:inequality-equivalent-complement}
Consider a linear inequality of the form
\begin{equation}\label{eq:lin-complement}
\alpha_1 \product w_1 + \dots + \alpha_n \product w_n < 0
\end{equation}
where the $w_i$'s are integer variables and the $\alpha_j$'s positive integer coefficients.
The inequality is satisfied iff $w_i \le 0$ for every $1 \leq i \leq n$ and $w_j <0$
for some $1 \leq j \leq n$.
\end{lemma} 
\begin{proof}
($\Leftarrow$) It follows straightforwardly from the fact that each $\alpha_j>0$ for every $j \in [1,n]$.

\noindent
($\Rightarrow$) By contradiction, assume that $(\ref{eq:lin-complement})$ is satisfied for every $\alpha_j>0$, where $j \in [1,n]$,
but either there is $i \in [1,n]$ such that $w_i>0$ or $w_i \le 0$ for every $i \in [1,n]$ but none of such inequalities is strict. 
If there is $i \in [1,n]$, ($i=1$, for example) such that $w_1 > 0$, then, since $\alpha_j>0$ for each $j \in [1,n]$, any assignment of
$\alpha_1,\dots,\alpha_n>0$ such that $\alpha_1 > |\alpha_2 \product w_2 + \dots + \alpha_n \product w_n|$ will not
satisfy $(\ref{eq:lin-complement})$. In the case whether no $w_i \le 0$ is strict, then $w_i=0$ for every $i \in [1,n]$ and thus
$\alpha_1  \product w_1 + \dots + \alpha_n \product w_n$ will be zero, which does not satisfy $(\ref{eq:lin-complement})$.
\hfill
\end{proof} 

The next result says that checking if a program $\PP$ is \gb\ is in $\co\NP$.
We recall that a given a set of linear constraints depending on some \nvars\ is satisfiable if there exist non-negative integer values of its \nvars\ that satisfy the constraints. A solution of such linear constraints is any assignment for their \nvars\ to some non-negative integer values satisfying the constraints.

\begin{theorem}\label{the:cycle-bounded-complexity}
Checking whether a program $\PP$ is \gb\ is in $\co\NP$.
\end{theorem}\label{the:cycle-bounded-complexity}
\begin{proof}
In order to prove the claim, we focus on the complement of our problem.
By definition, a program $\PP$ is not \gb\ if  there exists a linear version $\PP'$ of $\PP$ which is not \gb, which means that a relevant basic cyclic path 
$\pi=\<r_1,r_2\>\dots\<r_n,r_1\>$ of $\unig{\PP'}$ is
such that either $eq(\pi)$ is not satisfiable or there is a
solution of $eq(\pi)$ for which the inequality $\overline{\alpha}_p \product size(rbody(r_1)) - \overline{\alpha}_p \product size(head(r_n)) \ge 0$
is false, for every $\overline{\alpha}_p \in \mathbb{N}^{\ar(p)}$. Checking the statement above can be carried out by the following non-deterministic procedure.

Guess a linear version $\PP'$ of $\PP$ and a  basic cyclic path $\pi$ of $\unig{\PP'}$ and check
If $\pi$ is relevant. if it is not, then reject (i.e., the program is \gb). Then, check if $eq(\pi)$ is satisfiable, if it is not then accept (i.e., the program
is not \gb). 
Now, it remains to check  whether there is a solution of $eq(\pi)$ such that $\overline{\alpha}_p \product size(rbody(r_1)) - \overline{\alpha}_p \product size(head(r_n)) \ge 0$ is false for all $\overline{\alpha}_p \in \mathbb{N}^{\ar(p)}$. To accomplish the aforementioned task,
we can check wheher 
$\overline{\alpha}_p \product size(rbody(r_1)) - \overline{\alpha}_p \product size(head(r_n)) < 0$
is true. 
Moreover, isolating every term $\overline{\alpha}_p[i]$ ($1 \le i \le \ar(p)$) in the inequality, we get an expression of the form
$\overline{\alpha}_p[1] \product w_1 + \dots + \overline{\alpha}_p[\ar(p)] \product w_{\ar(p)}<0 $, where each $w_i$ depends only on variables 
occurring in $eq(\pi)$. 
Since from Lemma~\ref{lem:inequality-equivalent-complement}, this is equivalent to check 
whether $w_i \le 0$ for $i \in [1,n]$ and there is $j \in [1,n]$ such that $w_j<0$, 
checking whether there is a solution of $eq(\pi)$ such that 
$\overline{\alpha}_p \product size(rbody(r_1)) - \overline{\alpha}_p \product size(head(r_n))\ge0$ 
is false for all $\overline{\alpha}_p \in \mathbb{N}^{\ar(p)}$ is equivalent to
guessing a $j \in [1,n]$ and check that the set of linear constraints 
$eq(\pi) \cup \{ w_1 \le 0 \} \cup \cdots \cup \{w_j < 0\} \cup \cdots \cup \{w_n \le 0\}$ 
is satisfiable. 
The input program is not \gb\ iff the previous set of linear constraints is satisfiable.

To show the desired upper bound, note that guessing a linear version $\PP'$ of $\PP$ and a basic cyclic path of $\unig{\PP'}$ can be done in non-deterministic polynomial time, since $|\PP'|=|\PP|$ and the maximum length of a 
basic cyclic path coincides with the number of edges of $\unig{\PP'}$. Moreover, as previously stated, constructing the firing graph is feasible in deterministic polynomial time. Furthermore, the construction of $eq(\pi)$ can be carried on in polynomial time too, by using a polynomially sized representation of the mgu's
of the rules occurring in $\pi$ \cite{Venturini75}. Finally, as shown in \cite{Papadimitriou81}, checking whether the set of linear constraints 
$eq(\pi) \cup \{ w_1 \le 0 \} \cup \cdots \cup \{w_j < 0\} \cup \cdots \cup \{w_n \le 0\}$ 
is satisfiable is in $\NP$.~\hfill
\end{proof}

\section{Related Work}\label{sec:related-work}
  
A significant body of work has been done on termination of logic programs under top-down evaluation 
\cite{Schreye94,Voets11,Marchiori96,Ohlebusch01,CodishLS05,SerebrenikS05,NishidaV10,Schneider-KampGST09,Schneider-KampGSST10,NguyenGSS07,BruynoogheCGGV07,Bonatti04,BaseliceBC09}
and in the area of term rewriting~\cite{Zantema95,SternagelM08,ArtsG00,EndrullisWZ08,FerreiraZ96}.
Termination properties of query evaluation for normal programs under tabling have been studied in~\cite{RiguzziS13,RiguzziS13TR,VerbaetenSS01}.

In this paper, we consider logic programs with function symbols \emph{under the stable model semantics}~\cite{GelLif88,GelLif91} (recall that, as discussed in Section~\ref{sec:locally_bounded}, our approach can be applied to programs with disjunction and negation by transforming them into positive normal programs), and thus all the excellent works above cannot be straightforwardly applied to our setting---for a discussion on this see, e.g.,~\cite{CalimeriCIL08,AlvianoFL10}.
In our context, \cite{CalimeriCIL08} introduced the class of \emph{finitely-ground programs},
guaranteeing the existence of a finite set of stable models, each of finite size, for programs in the class.
Since membership in the class is not decidable, decidable subclasses have been proposed:
\emph{$\omega$-restricted programs},
\emph{$\lambda$-restricted programs},
\emph{finite domain programs},
\emph{argument-restricted programs},
\emph{safe programs},
\emph{$\Gamma$-acyclic programs},
\emph{mapping-restricted programs}, and 
\emph{bounded programs}.
An adornment-based approach that can be used in conjunction with the techniques above to detect more programs as finitely-ground has been proposed in~\cite{GrecoMT13iclp}.
This paper refines and extends~\cite{CalauttiGMT14}.

Compared with the aforementioned classes, \locally\ and \gb\ programs allow us to perform a more global analysis and identify many practical programs as terminating, such as those where terms in the body are rearranged in the head, which are not included in any of the classes above. 
We observe that there are also programs which are not \locally\ or \gb\ but are recognized as terminating by some of the aforementioned techniques (see Theorems~\ref{th:comparison-1}~and~\ref{th:comparison-2}).

Similar concepts of ``term size'' have been considered to check termination of logic programs evaluated in a top-down fashion~\cite{Sohn1991}, 
to check local stratification of logic programs \cite{Palopoli92}, in the context of partial evaluation to provide conditions for strong termination and quasi-termination~\cite{Vidal07,LeuschelVidal2014}, and in the context of tabled resolution~\cite{RiguzziS13,RiguzziS13TR}.
These approaches are geared to work under top-down evaluation, looking at how terms are propagated from the head to the body, while our approach is developed to work under bottom-up evaluation, looking at how terms are propagated from the body to the head.
This gives rise to significant differences in how the program analysis is carried out, making one approach not applicable in the setting of the other.
As a simple example, the rule $\tt p(X) \leftarrow p(X)$ leads to a non-terminating top-down evaluation, while it is completely harmless under bottom-up evaluation.

We conclude by mentioning that our work is also related to research done on termination of the chase procedure, where existential rules are considered
~\cite{Marnette-09,SpezzanoG10,GrecoST11}; a survey on this topic can be found in~\cite{Morgan12Greco}.
Indeed, sufficient conditions ensuring termination of the bottom-up evaluation of logic programs can be directly applied to existential rules.
Specifically, one can analyze the logic program obtained from the skolemization of existential rules, where existentially quantified variables are replaced with complex terms~\cite{Marnette-09}.
In fact, the evaluation of such a program behaves as the ``semi-oblivious'' chase~\cite{Marnette-09}, whose termination guarantees the termination of the standard chase~\cite{MeierThesis,Onet13}.

\section{Conclusions}\label{sec:conclusions}

As a direction for future work, we plan to investigate how our techniques can be combined with current termination criteria in a uniform way.
Since they look at programs from radically different standpoints,  
an interesting issue is to study how they can be integrated so that they can benefit from each other.

To this end, an interesting approach would be to plug termination criteria in the generic framework proposed in~\cite{EiterFKR13} and study their combination in such a framework. Another intriguing issue would be to analyze the relationships between the notions of safety of~\cite{EiterFKR13} and the notions of boundedness used by termination criteria.

\bibliographystyle{plain}
\bibliography{ProgramTermination}

\end{document}